\newif\ifarxiv\arxivtrue
\newif\ifcondensation\condensationfalse
\newif\ifear\earfalse
\newif\ifbiconn\biconnfalse
\ifarxiv
\let\shortcite\cite
\documentclass{article}
\else
\documentclass[letterpaper]{article} 
\usepackage{aaai23}  
\fi
\usepackage{times}  
\usepackage{helvet}  
\usepackage{courier}  
\usepackage[hyphens]{url}  
\usepackage{graphicx} 
\ifarxiv
\else
\usepackage{natbib}  
\fi
\usepackage{caption} 
\def\circledarrow#1#2#3{ 
\draw[#1,->] (#2) +(100:#3) arc(100:420:#3);
}
\def\backcircledarrow#1#2#3{ 
\draw[#1,->] (#2) +(420:#3) arc(420:100:#3);
}
\usepackage{amsfonts}
\usepackage[fleqn]{mathtools}
\usepackage{cleveref}
\usepackage{subfigure}
\usepackage{algorithm}
\usepackage{algpseudocode}
\usepackage{tikz}
\usepackage{amsthm}
\usepackage{amssymb}
\usepackage{enumitem}
\usepackage{wasysym}
\usepackage{url}
\usepackage[stable]{footmisc}

\newtheorem{theorem}{Theorem}
\newtheorem{lemma}[theorem]{Lemma}
\newtheorem{proposition}[theorem]{Proposition}
\newtheorem{corollary}[theorem]{Corollary}
\newtheorem{hyp}[theorem]{Hypothesis}

\title{The Small Solution Hypothesis for MAPF on Strongly Connected
  Directed Graphs Is True}
\author{Bernhard Nebel
  \ifarxiv%
Albert-Ludwigs-Universität \\
  Freiburg, Germany \\
  nebel@uni-freiburg.de
\fi%
}
\ifarxiv
\else
\affiliations{Albert-Ludwigs-Universität \\
  Freiburg, Germany \\
nebel@uni-freiburg.de}
\fi

\begin{document}
\maketitle

\begin{abstract}
  The determination of the computational complexity of multi-agent
  pathfinding on directed graphs  (diMAPF) has been an open research problem for
  many years.  While diMAPF has been shown to be
  polynomial for some special cases, it has only recently been established that the problem is
  NP-hard in general. Further, it has been proved that diMAPF will be in NP if
  the short solution hypothesis for strongly connected directed
  graphs is correct. In this paper, it is shown that this hypothesis is indeed
  true, even when one allows for synchronous rotations.
\end{abstract}

\section{Introduction}

Multi-agent pathfinding (MAPF), often also called pebble motion on
graphs or
cooperative pathfinding, is the problem of deciding
the existence of or generating a collision-free movement plan for a set of agents moving on a
graph, most often a graph generated from a grid, where agents can move to
adjacent grid cells
\cite{ma:koenig:aim-17}. Two examples are provided in Figure~\ref{F:agents}.
\begin{figure}[htb]
\centering
\resizebox{\ifarxiv%
0.7\columnwidth%
\else%
\columnwidth\fi}{!}{
\begin{tikzpicture}
\draw (0,0) rectangle +(2,2);
\draw(2,0) rectangle +(2,2);
\draw(4,0) rectangle +(2,2);
\draw(2,-2) rectangle +(2,2);
\node at (1,1.1) {\includegraphics[scale=0.075]{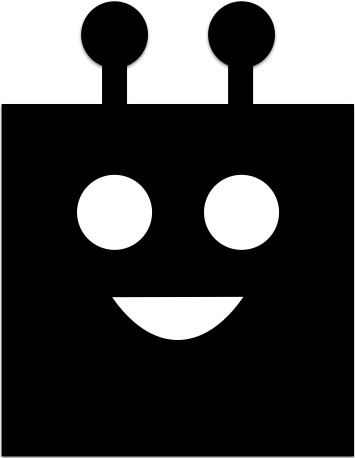}};
\node at (3,-1.1) {\includegraphics[scale=0.075]{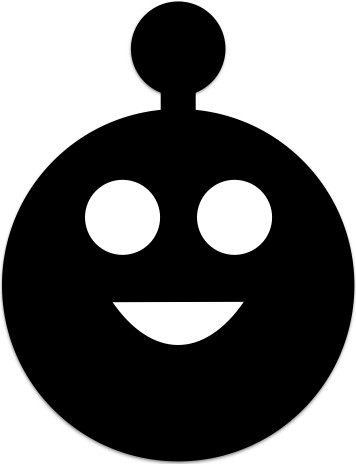}};
\node at (0.3,1.75){\Large $v_1$};
\node at (2.3,1.75){\Large $v_2$};
\node at (4.3,1.75){\Large $v_3$};
\node at (2.3,-0.25){\Large $v_4$};
\node at (3.7,+0.3) [circle,fill=black,minimum size=10pt]{};
\node at (5.7,+0.3) [rectangle,fill=black,minimum size=10pt]{};
\end{tikzpicture} \qquad
\begin{tikzpicture}
\draw (0,0) rectangle +(2,2);
\draw(2,0) rectangle +(2,2);
\draw(4,0) rectangle +(2,2);
\draw(2,-2) rectangle +(2,2);
\node at (1,1.1) {\includegraphics[scale=0.075]{robots}};
\node at (3,-1.1) {\includegraphics[scale=0.075]{robotc}};
\node at (5,1.1) {\includegraphics[scale=0.075]{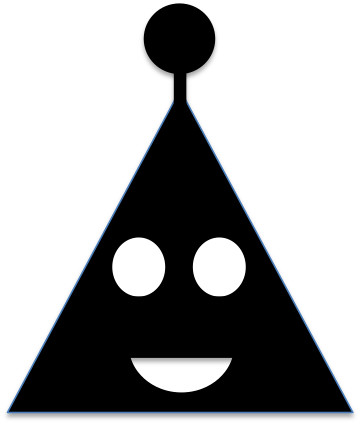}};
\node at (0.3,1.75){\Large $v_1$};
\node at (2.3,1.75){\Large $v_2$};
\node at (4.3,1.75){\Large $v_3$};
\node at (2.3,-0.25){\Large $v_4$};
\node at (3.7,+0.3) [circle,fill=black,minimum size=10pt]{};
\node at (5.7,+0.3) [rectangle,fill=black,minimum size=10pt]{};
\node at (1.7,+0.3){\includegraphics[scale=0.15]{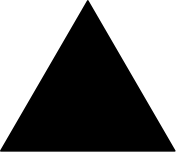}};
\end{tikzpicture}}
\caption{Multi-agent pathfinding examples}
\label{F:agents}
\end{figure}

In the left example, the circular agent $C$ needs to move to $v_2$ and the square
agent $S$ has to move to $v_3$. $S$ could move first to $v_2$ and then to
$v_3$, after which $C$ could move to its destination $v_2$.
So, in
this case, a collision-free movement plan exists.  In the right example, where
additionally the triangle agent has to move to $v_1$, there is no
possible way for the square and triangle agent to exchange their
places, i.e., there does not exist any collision-free movement plan.

Kornhauser et al. \shortcite{kornhauser:et:al:focs-84} had shown
already in the eighties that deciding MAPF is a polynomial-time
problem and movement plans have polynomial length, although it took a
while until this result was recognized in the community
\cite{roeger:helmert:socs-12}.  The optimizing variant of this
problem, assuming that only one agent can move at each time step, had
been shown to be NP-complete soon after the initial result
\cite{goldreich:manuscript-84,ratner:warmuth:aaai-86}.

Later on, variations of the problem have been studied
\cite{felner:et:al:socs-17}. It is obvious that all robots that move
to different empty nodes could move in parallel, which may lead to
shorter plans. Assuming coordination between the agents, one
can also consider train-like movements, where only the first robot
moves to an empty node and the others follow in a chain, all in one
time step \cite{ryan:jair-08,surynek:ictai-09,surynek:aaai-10}. Taking
this one step further, synchronous rotations of agents on a cycle
without any empty nodes have
been considered
\cite{standley:aaai-10,yu:lavalle:aaai-13,yu:rus:wafr-14}. And even
distributed and epistemic versions have been studied
\cite{nebel:et:al:jair-19}.

Concerning plan existence and polynomial plan length, parallel and
train-like movements do not make a difference to the case when only
simple moves are permitted. Synchronous rotations are a different
story altogether. There are problem instances which cannot be solved
using only simple moves, but are  solvable when
synchronous rotations are allowed. Nevertheless, it has been shown that MAPF is a
polynomial-time problem as well \cite{yu:rus:wafr-14}. Distributed and
epistemic versions are more difficult, however. They are NP-complete
and PSPACE-complete, respectively \cite{nebel:et:al:jair-19}.

Optimizing wrt.{} different criteria turned out to be NP-complete
\cite{surynek:aaai-10,yu:lavalle:aaai-13} for all kinds of movements,
and this holds even for planar and grid graphs
\cite{yu:ieee-ral-16,banfi:et:al:ieee-ral-17,geft:halperin:aamas-22}.
Additionally, it was shown that there are limits to the
approximability of the optimal solution for makespan optimizations
\cite{ma2016multi}.

The mentioned results all apply to undirected graphs only. However,
a couple of years ago, researchers also considered MAPF on directed
graphs (diMAPF) \cite{wang:botea:icaps-08}, and it has been proved
that diMAPF can be decided in polynomial time,
provided the directed graph is strongly biconnected\ifbiconn\else\footnote{This means that
  each pair of nodes is located on a directed cycle.}\fi{}
 and there are at least two
unoccupied vertices \cite{botea:surynek:aaai-15,botea:et:al:jair-18}. The general case for directed graphs
is still open, though.  It has only been shown that diMAPF is NP-hard
\cite{nebel:icaps-20}, but membership in NP was not established.

In general, the state space of (di)MAPF has size $O(n!)$, $n$ being
the number of vertices of the graph. However, for directed acyclic
graphs, a quadratic upper bound for the plan length is obvious,
because steps are not reversible in this case, and so each single
agent can perform at most $n$ steps. This leads immediately to
NP-completeness for directed acyclic graphs.  For strongly connected
directed graphs, this argument is not valid, however.  It was
nevertheless conjectured that plans can be polynomially bounded
\cite{nebel:icaps-20}, since this holds also for undirected graphs and
for directed strongly biconnected graphs with two empty nodes.  In
this paper, we will show that this \emph{small solution hypothesis} is
indeed true.

The crucial observation is that each movement in a strongly connected
directed graph can be reversed, which essentially means that movements
of single agents can also be thought of taking place against the direction of
an arc in directed graphs, something already noted by Ardizzoni et
  al. \shortcite{ardizzoni:et:al:cdc-22}.
This implies that we can reuse almost all
results about undirected graphs with only a polynomial overhead
in plan length---provided we are talking about single-agent movements.
However, this does not apply to synchronous
rotations. Here, a reduction to movements on the corresponding
undirected graph is impossible. By
using techniques similar to the ones Kornhauser et
al. \shortcite{kornhauser:et:al:focs-84} employed, we show that any
movement plan using only synchronous rotations on strongly connected
directed graphs can be polynomially bounded. Finally, we consider the
case, when both kinds of movements are allowed, and prove a polynomial
bound as well, which allows us
to conclude that the small solution hypothesis is indeed
true regardless of what kind of movements are possible, and therefore diMAPF is NP-complete.

The rest of the paper is structured as follows. In the next section,
we will introduce the necessary notation and terminology that is
needed for proving the small solution hypothesis to be true. We will
cover basic notation and terminology from graph theory and permutation groups, and
will formally introduce MAPF and diMAPF. In the three subsequent
sections, the small solution hypothesis will then be proven, first for
the case of simple movements, then for the more complicated case when
only synchronous rotations are permitted, and finally for the case, when
both kinds of movements are possible.
The paper ends with a short
discussion of the results.

\section{Notation and Terminology}

\subsection{Graph Theory}

A {\em graph} $G$ is a tuple $(V,E)$ with $E \subseteq \{ \{u,v\} \mid
u,v \in V\}$. The elements of $V$ are called \emph{nodes} or  {\em vertices} and the
elements of $E$ are called {\em edges}. A {\em directed graph} or {\em
  digraph} $D$ is a tuple $(V,A)$ with $A \subseteq V^2$. The elements
of $A$ are called {\em arcs}.
Graphs and digraphs
with $|V| = 1$ are called \emph{trivial}.
We
assume all graphs and digraphs to be {\em simple}, i.e., not containing any
self-loops of the form $\{u\}$, resp. $(u,u)$.
Given
a digraph $D=(V,A)$,  the {\em underlying graph} of $D$, in symbols
${\cal G}(D)$, is the graph resulting from ignoring the direction of
the arcs, i.e., ${\cal G}(D) = (V, \{\{u,v\} \mid (u,v) \in A\})$. 

Given a graph $G=(V,E)$,  $G'=(V',E')$ is called {\em
  sub-graph} of $G$ if $V \supseteq V'$ and $E \supseteq E'$.
Similarly, for
digraphs $D=(V,A)$ and $D'=(V',A')$, $D'$ is a \emph{sub-digraph} if
$V \supseteq V'$ and $A \supseteq A'$.

A {\em path} in a 
graph $G=(V,E)$ is a non-empty sequence of vertices  of the form
$v_0, v_1,\ldots, v_k$ such that $v_i \in V$ for all
$0 \leq i \leq k$, $v_i \neq v_j$ for all $0 \leq i < j \leq k$, and 
$\{v_{i}, v_{i+1}\} \in E$ for all $0 \leq i < k$. A {\em cycle} in a
graph $G=(V,E)$ is a non-empty sequence of
vertices $v_0,v_1,\ldots,v_k$ with $k \geq 3$ such that $v_0 = v_k$,
$\{v_i,v_{i+1}\} \in E$ for all
$0 \leq i < k$ and $v_i \neq v_j$ for all $0 \leq i < j < k$. 

In a digraph $D=(V,A)$, {\em path} and {\em cycle} are similarly
defined, except that the direction of the arcs has to be
respected. This means that $(v_{i}, v_{i+1}) \in A$ for all $0 \leq i < k$.
Further, the smallest cycle in a digraph has 2 nodes instead of 3.
A digraph that does not contain any cycle is called {\em directed acyclic
  graph (DAG)}. A digraph that consists solely of a cycle is
called \emph{cycle digraph}. A digraph that consists of a directed cycle
$v_0, v_1, \ldots, v_{k-1}, v_k=v_0$ and any number of arcs connecting adjacent
nodes in a backward manner, i.e., $(v_i,v_{i-1}) \in A$, is called \emph{partially
  bidirectional cycle graph}. 

A graph $G=(V,E)$ is {\em connected} if there is a path between each
pair of vertices.
%
A connected graph that does not contain a cycle is called
{\em tree}. 

Similarly, a digraph $D=(V,A)$ is {\em weakly
  connected}, if the underlying graph ${\cal G}(D)$ is connected. It
is {\em strongly connected}, if for every pair of vertices $u,v$, there
is a path in $D$ from $u$ to $v$ and one from $v$ to $u$.
%
The smallest
strongly connected {} digraph is
the trivial digraph.
%
%
The {\em strongly connected components} of a digraph $D=(V,A)$ are
the maximal sub-digraphs $D_i=(V_i,A_i)$ that are strongly
connected.

\subsection{Permutation Groups}

In order to be able to establish a polynomial bound for diMAPF
movement plans containing synchronous rotations, we introduce some background on
permutation groups.\footnote{A gentle introduction to the topic is the book by Mulholland
\shortcite{mulholland:book-21}.}

A \emph{permutation} is a bijective function over a set $X$ 
$\sigma\colon X \rightarrow X$. In what follows, we assume $X$ to be
finite.

A permutation has \emph{degree} $d$ if it exchanges $d$ elements and
fixes the rest of the the elements in $X$.
We say that a permutation is an
\emph{$m$-cycle} if it exchanges elements $x_1, \dots, x_m$ in a
cyclic fashion, i.e., $\sigma(x_i) = x_{i + 1}$ for $1 \leq i < m$,
$\sigma(x_m) = x_1$ and $\sigma(y) = y$ for all
$y \notin \{ x_i \}_{i = 1}^m$. Such a cyclic permutation is written
as a list of elements, i.e., $(x_1\ x_2\ \cdots\ x_m)$. A $2$-cycle is
also called \emph{transposition}.
A permutation
can also consist of different disjoint cycles. These are then written
in sequence, e.g., $(x_1\ x_2) (x_3\ x_4)$. 

The composition of two permutations $\sigma$ and $\tau$, written as
$\sigma \circ \tau$ or simply $\sigma\tau$, is the function mapping $x$ to
$\tau(\sigma(x))$.\footnote{Note that this order of function applications, which is used
  in the context of permutation groups, is different
  from ordinary function composition.} This operation is
associative because function composition is.
The special permutation
$\epsilon$, called \emph{identity}, maps every element to
itself. Further, $\sigma^{-1}$ is the \emph{inverse} of $\sigma$, i.e.,
$\sigma^{-1}(y) = x$ if and only if $\sigma(x) = y$. The $k$-fold
composition of $\sigma$ with itself is written as
$\sigma^k$, with $\sigma^0 = \epsilon$. Similarly, $\sigma^{-k} := (\sigma^{-1})^{k}$
We also consider the \emph{conjugate of $\sigma$ by $\tau$},
written as $\sigma^{\tau}$, which is defined to be
$\tau^{-1} \sigma \tau$. Such conjugations are helpful in creating new
permutations out of existing ones. 
We use exponential notation as in
the book by Mulholland \shortcite{mulholland:book-21}:
 $\sigma^{\alpha + \beta} := \sigma^\alpha \sigma^\beta$ and
$\sigma^{\alpha\beta} := (\sigma^{\alpha})^{\beta}$.

A set of permutations closed under composition and inverse forms a
\emph{permutation group} with $\circ$ as the \emph{product operation}
(which is associative), $\cdot^{-1}$ being the \emph{inverse
  operation}, and $\epsilon$ being the \emph{identity element}. Given a
set of permutations $\{ g_1,\ldots, g_i\}$, we say that
$\mathbf{G} = \langle g_1,\ldots, g_i\rangle$ is \emph{the permutation
  group generated by $\{ g_1,\ldots, g_i\}$} if $\mathbf{G}$ is the
group of permutations that results from product operations over the
elements of $\{ g_1,\ldots, g_i\}$. We say that
$\sigma \in \mathbf{G} = \langle g_1,\ldots, g_i\rangle$ is
$k$-expressible if it can be written as a product over the generators
using $< k$ product operations.  The \emph{diameter} of a group
$\mathbf{G} = \langle g_1,\ldots, g_i\rangle$ is the least number $k$
such that every element of $\mathbf{G}$ is $k$-expressible. Note that
this number depends on the generator set.

A group ${\mathbf F}$ is a
subgroup of another group ${\mathbf G}$, written ${\mathbf F} \leq {\mathbf G}$,  if the elements of
${\mathbf F}$ are a subset of the elements of ${\mathbf G}$.
In our context, two permutation groups are of particular interest. One
is ${\mathbf{S}}_n$, the \emph{symmetric group} over $n$ elements,
which consists of all permutations over $n$ elements. A permutation in ${\mathbf{S}}_n$
is \emph{even} if it can be represented as a product of an even number
of transpositions, \emph{odd} otherwise. Note that 
no permutation can be
odd and even at the same time \cite[Theorem~7.1.1]{mulholland:book-21}.
The set of even permutations
forms another group, the \emph{alternating group} ${\mathbf{A}}_n \leq
{\mathbf{S}}_n$. Since any $m$-cycle can
be equivalently expressed as the product of $m-1$ transpositions, $m$-cycles with even $m$ are not elements of ${\mathbf{A}}_n$.

\begin{sloppy}
A permutation group ${\mathbf{G}}$ is \emph{$k$-transitive}
if for all pairs of $k$-tuples
$(x_1, \ldots, x_k),$ $(y_1, \ldots, y_k)$, there exists a permutation
$\sigma \in {\mathbf{G}}$ such that
$\sigma(x_i) = y_i, 1 \leq i \leq k$. In case of $1$-transitivity we
simply say that $ {\mathbf{G}}$ is transitive.
\end{sloppy}

A \emph{block} is a non-empty subset $B \subseteq X$ such that for each
permutation $\sigma$, either $\sigma(B) = B$ or $\sigma(B) \cap
B = \emptyset$. Singleton sets and the entire set $X$ are
\emph{trivial blocks}. Permutation groups that contain only trivial
blocks are \emph{primitive}.

\subsection{Multi-Agent Pathfinding}

Given a graph $G =(V,E)$ and a set of {\em agents} $R$ such that
$|R| \leq |V|$, we say that the injective function
$S\colon R \rightarrow V$ is a \emph{multi-agent pathfinding (MAPF)
  state} (or simply {\em state}) over $R$ and $G$. Any node not
occupied by an agent, i.e., a node $v \in V - S(R)$, is called
\emph{blank}.

A {\em multi-agent pathfinding (MAPF) instance} is then a tuple
$\langle G,R,I,T\rangle$ with $G$ and $R$ as above and $I$ and $T$ 
MAPF states. A \emph{simple move} of agent $r$ from node $u$ to node $v$, in
symbols $m=\langle r,u,v\rangle$, transforms a given state $S$, where we
need to have $\{u,v\} \in E$, $S(r) = u$ and $v$ is a blank, into the successor state
$S^{[m]}$, which is identical to $S$ except
at the point $r$, where $S^{[m]}(r) = v$.

If there exists a (perhaps empty) sequence of simple moves, a
\emph{movement plan}, that transforms
$S$ into $S'$, we say that $S'$ is \emph{reachable} from $S$.  The
MAPF problem is then to decide whether $T$ is reachable from $I$.

The MAPF problem is often defined in terms of parallel or train-like movements
\cite{ryan:jair-08,surynek:aaai-10}, where one step consists of
parallel non-interfering moves of many agents. However, as long as we
are interested only in solution existence and polynomial bounds, there
is no difference between the MAPF problems with parallel and
sequential movements.  If we allow for {\em synchronous rotations}
\cite{standley:aaai-10,yu:lavalle:aaai-13,yu:rus:wafr-14}, where
one assumes that all agents in a fully occupied cycle can move
synchronously, things are a bit different. In this case, even if there
are no blanks, agents can move.  A \emph{rotation} is a set of simple
moves $M= \{ \langle r_1, u_1, v_1 \rangle,$ $\ldots$,
$\langle r_k, u_k, v_k \rangle\}$, where $k \geq 3$, with (1)
$\{u_i,v_i\} \in E$ for $1 \leq i \leq k$, (2) $u_i \neq u_j$ for
$1 \leq i < j \leq k$, (3) $v_{i} = u_{i+1}$ for $1 \leq i < k$, and
$v_k = u_1$. Such a rotation $M$ is executable in $S$ if
$S(r_i) = u_i$ for $1 \leq i \leq k$.  The successor state $S^{[M]}$ of a
given state $S$ is identical to $S$ except at the points
$r_1, \ldots, r_k$, where we have $S^{[M]}(r_i) = v_i$.

{\em Multi-agent pathfinding on directed graphs (diMAPF)} is similar
to MAPF, except that we have a directed graph and the moves have to
follow the direction of an arc, i.e., if there is an arc $(u,v) \in A$
but $(v,u) \not\in A$, then an agent can move from $u$ to $v$ but not
vice versa. Since on directed graphs there are also cycles of size
two, one may also allow rotations on only two nodes, contrary to the
definition of rotations on undirected graphs. In the following, we
consider both possibilities, allowing or prohibiting rotations of size
2, and show that it does not make a difference when it comes to bounding
movement plans polynomially. 

The hypothesis that we want to prove correct can now be formally stated.

\begin{hyp}[Short Solution Hypothesis for diMAPF on strongly connected
  digraphs]
  For each solvable diMAPF instance on strongly connected digraphs,
  there exists a movement plan of polynomial length.
\end{hyp}

When this hypothesis holds, then diMAPF is in NP, as follows from
Theorem 4 in the paper that established NP-hardness of diMAPF
\cite{nebel:icaps-20}.
The short justification for this
is that in each movement plan on a digraph, each agent only linearly
often enters and leaves strongly connected components. If the
movements inside each component can be polynomially bounded, then the
overall plan is polynomially bounded, i.e., can be guessed in
polynomial time. Together with the NP-hardness result  \cite[Theorem
1]{nebel:icaps-20}, it follows that diMAPF is NP-complete.

\section{DiMAPF Without Rotations}

As mentioned in the Introduction, one crucial observation is that any
simple move on strongly connected components can be undone. 
When considering cycle graphs, it is
obvious that it is always possible to restore a previous state.

\begin{proposition}
  \label{P:undo}
Let $S$ be a diMAPF state over the set of agents $R$ and the cycle
digraph $D=(V,A)$, and let $S'$ be a state reachable from $S$ with
simple moves, then one can reach
$S$ from $S'$ in at most $O(|V|^2)$ simple moves.
\end{proposition}

\begin{proof}
  Since the relative order of agents on a cycle cannot be changed by
  movements, one can always reach the initial state, regardless of
  what movements have been made in order to deviate from the initial
  state. Further, the maximal distance of an agent from its position
  in $S$ is $|V|-1$, which is the maximal number of moves the agent
  has to make to reach $S$. Since there are at most $|V|$ agents, the
  stated upper bound follows.
\end{proof}

If we now consider a simple move on a strongly
connected digraph, then it is obvious that we can restore the original
state, because all movements in such a graph take place on a cycle.

\begin{proposition}
  \label{P:inverse:move}
Let $S$ be a diMAPF state over $R$ and a strongly connected digraph  $D=(V,A)$ and
let $m=\langle r,u,v\rangle$ be a simple move to transform $S$ into
$S^{[m]}$. Then there exists a plan consisting of simple moves of length $O(|V|^2)$ to reach $S$ from
$S^{[m]}$. 
\end{proposition}

\begin{proof}
  Each move $m$ in a strongly connected digraph is a move on a cycle in
  the digraph. Hence, we can apply Proposition~\ref{P:undo} and restore
  the original state $S$ using $O(v^2)$ simple moves, provided $v$ is
  the number of nodes in the cycle. Note that by restoring the
  configuration on the cycle, we restore the entire state. Further,
  since $v \leq |V|$, the claim about the plan length follows.  
\end{proof}

The plan of restoring the state before move $m$ is executed can actually be
seen as the ``inverse move'' $m^{-1}$, which moves an agent against
the direction of an arc! 

Although such a ``synthesized'' move against the arc direction is
costly---it may involve $O(|V|^2)$ simple moves---this opens up the
possibility to view a diMAPF instance on the digraph $D$ as a MAPF
instance on the underlying undirected graph ${\cal G}(D)$.\footnote{This had
already been noted by Ardizzoni et
al. \shortcite[Theorem~4.3]{ardizzoni:et:al:cdc-22}. However, the
proof appears to be incomplete, and the implication for plan length
had not been stated.} Because of the existence of inverse moves for
each possible simple move with only polynomial overhead, the next corollary is immediate.

\begin{corollary}
  \label{C:reduction}
  Let $\langle D,R,I,T\rangle$ be a diMAPF instance with $D$ a strongly
  connected digraph. Then the MAPF instance $\langle {\cal G}(D),R,I,T\rangle$ has a
  polynomial movement plan consisting of simple moves if and only if $\langle D,R,I,T\rangle$
  has a polynomial movement plan consisting of simple moves.
\end{corollary}

Using Corollary~\ref{C:reduction} together with the result about an upper bound for
the plan length on undirected graphs from the paper by Kornhauser et
al. \shortcite[Theorem~2]{kornhauser:et:al:focs-84} gives us the first partial
result on the small solution hypothesis.

\begin{theorem}
  \label{T:ssh:simple}
The {\em Small Solution Hypothesis for diMAPF on strongly connected
  digraphs} is true, provided only simple moves are allowed.
\end{theorem}

\section{DiMAPF With Rotations Only}

As a next step, we will consider the case that the only kind of
movements are rotations. In this case, we can
unfortunately not reduce the reachability on the directed graph to
reachability on the underlying undirected graph. In order to see the
problem, consider the directed graph in Figure~\ref{F:basic}(a). The
cycle in the underlying undirected graph formed by the nodes 
$c_2$, $c_3$, $c_4$, $e_2$, $e_1$, $c_2$  could be used for a rotation on
the undirected graph, but there is no obvious way to emulate such a
rotation on the directed graph. In fact, there exist unsolvable diMAPF
instances such that the corresponding MAPF instance on the underlying
undirected graph is solvable.
For this reason, we will employ
permutation group theory in order to derive a polynomial upper
bound for plan length in this case.

\subsection{Cycle pairs}

Firstly, let us have a closer look at the structure of strongly
connected non-trivial digraphs that are not partially bidirectional cycle digraphs.
These can be decomposed into a \emph{basic
cycle} and one or more \emph{ears} \cite[Theorem~5.3.2]{bang-jensen:gutin:book-09}. This means that each such graph contains at least a sub-graph
consisting of a directed cycle  (such as, e.g., $c_1$, $c_2$, $c_3$, $c_4$,
$c_5$, $c_1$  in Figure~\ref{F:basic}, which is drawn
solidly), and an ear (such as, e.g., $c_2$,  $e_1$, 
$e_2$, $c_4$ drawn in a dotted
way). An \emph{ear} is  a directed
path or cycle that starts at some node of the basic cycle (in Figure~\ref{F:basic}(a), $c_2$) and ends at a 
node of the basic cycle ($c_4$), which however could be the same node. The ear could either be
oriented in the same direction as the basic cycle, providing a detour or
short-cut as in Figure~\ref{F:basic}(a), or it points back, as  in
Figure~\ref{F:basic}(b).

\begin{figure}[htb]
\centering
  \resizebox{\ifarxiv%
0.8\columnwidth%
\else%
\columnwidth\fi}{!}{
\begin{tikzpicture}[thick, main/.style = {draw, circle}] 
  \node[main] (c1) at (0,2) {};
  \node (lc1) [above of=c1,yshift=-6mm] {$c_1$};
  \node[main] (c5) at (0,0) {};
  \node (lc5) [below of=c5,yshift=6mm] {$c_{5}$};
  \node[main] (c2) at (2,2) {};
  \node (lc2) [above of=c2,yshift=-6mm] {$c_2$};
  \node[main] (c3) at (2,1) {};
  \node (lc3) [above right of=c3,yshift=-4mm,xshift=-3mm] {$c_3$};  
  \node[main] (c4) at (2,0) {};
  \node (lc4) [below of=c4,yshift=6mm] {$c_4$};
  \node[main,dotted] (e1) at (4,2) {};
  \node (le1) [above of=e1,yshift=-6mm] {$e_1$};
  \node[main,dotted] (e2) at (4,0) {};
  \node (le2) [below of=e2,yshift=6mm] {$e_{2}$};
  \draw[->] (c1) -- (c2);
  \draw[->] (c2) -- (c3);
  \draw[->] (c3) -- (c4);
  \draw[->] (c4) -- (c5);
  \draw[->] (c5) -- (c1);
  \draw[->,dotted] (c2) -- (e1);
  \draw[->,dotted] (e1) -- (e2);
  \draw[->,dotted] (e2) -- (c4);
  \node (a) [below of=c4] {(a)};
\end{tikzpicture}\qquad
\begin{tikzpicture}[thick, main/.style = {draw, circle}] 
  \node[main] (c1) at (0,2) {};
  \node (lc1) [above of=c1,yshift=-6mm] {$c_1$};
  \node[main] (c5) at (0,0) {};
  \node (lc5) [below of=c5,yshift=6mm] {$c_{5}$};
  \node[main] (c2) at (2,2) {};
  \node (lc2) [above of=c2,yshift=-6mm] {$c_2$};
  \node[main] (c3) at (2,1) {};
  \node (lc3) [above right of=c3,yshift=-4mm,xshift=-3mm] {$c_3$};  
  \node[main] (c4) at (2,0) {};
  \node (lc4) [below of=c4,yshift=6mm] {$c_4$};
  \node[main,dotted] (e1) at (4,2) {};
  \node (le1) [above of=e1,yshift=-6mm] {$e_1$};
  \node[main,dotted] (e2) at (4,0) {};
  \node (le2) [below of=e2,yshift=6mm] {$e_{2}$};
  \draw[->] (c1) -- (c2);
  \draw[->] (c2) -- (c3);
  \draw[->] (c3) -- (c4);
  \draw[->] (c4) -- (c5);
  \draw[->] (c5) -- (c1);
  \draw[<-,dotted] (c2) -- (e1);
  \draw[<-,dotted] (e1) -- (e2);
  \draw[<-,dotted] (e2) -- (c4);
  \node (a) [below of=c4] {(b)};
\end{tikzpicture}
}
\caption{Strongly connected digraphs consisting of a cycle and an ear}
\label{F:basic}
\end{figure}
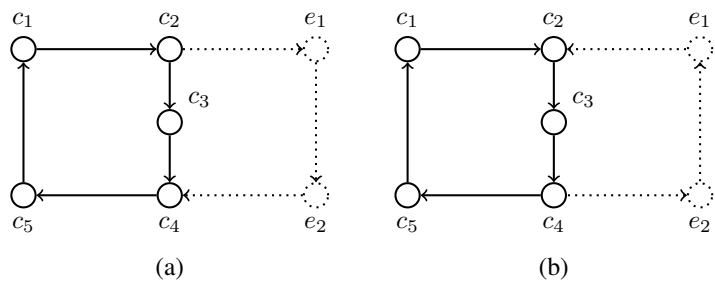

In order to be able to deal with only one kind of cycle pair, it is
always possible to
view a graph as in Figure~\ref{F:basic}(a) as one in
Figure~\ref{F:basic}(b). This can be accomplished by considering the
outer, larger cycle as the basic cycle and the path with $c_2$, $c_3$,
$c_4$ as an ear, as illustrated in Figure~\ref{F:transformed}. Note
that in the extreme one may have an ear with no additional nodes.

\begin{figure}[htb]
\centering
\resizebox{\ifarxiv%
0.4\columnwidth%
\else%
0.5\columnwidth\fi}{!}{
\begin{tikzpicture}[thick, main/.style = {draw, circle}] 
  \node[main] (c1) at (2,1.4) {};
  \node (lc1) [above right of=c1,yshift=-5mm,xshift=-3mm] {$c_1$};
  \node[main] (c5) at (2,0.7) {};
  \node (lc5) [above right of=c5,yshift=-5mm,xshift=-3mm] {$c_{5}$};
  \node[main] (c2) at (2,2) {};
  \node (lc2) [above of=c2,yshift=-6mm] {$c_2$};
  \node[main,dotted] (c3) at (0,1) {};
  \node (lc3) [above of=c3,yshift=-6mm] {$c_3$};  
  \node[main] (c4) at (2,0) {};
  \node (lc4) [below of=c4,yshift=6mm] {$c_4$};
  \node[main] (e1) at (4,2) {};
  \node (le1) [above of=e1,yshift=-6mm] {$e_1$};
  \node[main] (e2) at (4,0) {};
  \node (le2) [below of=e2,yshift=6mm] {$e_{2}$};
  \draw[->] (c1) -- (c2);
  \draw[->,dotted] (c2) -- (c3);
  \draw[->,dotted] (c3) -- (c4);
  \draw[->] (c4) -- (c5);
  \draw[->] (c5) -- (c1);
  \draw[->] (c2) -- (e1);
  \draw[->] (e1) -- (e2);
  \draw[->] (e2) -- (c4);
\end{tikzpicture}
}
\caption{Different perspective on graph from Figure~\ref{F:basic}(a)}
\label{F:transformed}
\end{figure}
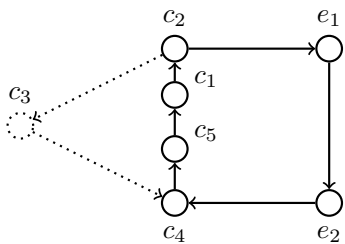

This means that in a strongly connected digraph which is not a
bidirectional cycle, one can always find two connected directed cycles
of a particular form
as stated in the following proposition.

\begin{proposition}
  \label{P:normalize:SBC}
  Any strongly connected non-trivial digraph that is not a partially bidirectional
  cycle graph contains two directed cycles that share at least one
  node, where at least one cycle contains more nodes than the shared
  ones. Further, if one cycle contains only shared nodes, then there
  are at least three shared nodes.
\end{proposition}

\begin{proof}
  There at least two cycles, since otherwise it is a partially bidirectional
  cycle graph. At least two of these cycle need to share at least one
  node, since the graph is strongly connected. One of these cycles
  needs to contain more than the shared nodes because otherwise
  we would not have two cycles. Finally, if one cycle contains only
  the shared nodes, then there needs to be at least three shared
  nodes, because otherwise the structure would be a partially bidirectional
cycle (in case of two shared nodes), or a non-simple digraph (in case
of only one shared node), which we excluded above.
\end{proof}

\subsection{Rotations as Permutations}

Rotations on cycles in the digraph will now be viewed as
permutations. Note that such permutations are cyclic permutations. If
we refer to such a cyclic permutation of degree $m$, we may call them
\emph{$m$-cycle}. If we refer to a graph cycle consisting of $n$
nodes, we will write \emph{cycle of size $n$}.  

Given a diMAPF instance $\langle
D,R,I,T\rangle$ on a strongly connected graph $D$ with no blanks, we
will view the sequence of rotations that transforms $I$ into $T$
as a sequence of permutations on $V$ that when composed permutes $I$ into
$T$. The \emph{permutation group rotation-induced by such an
instance} will also often be called \emph{permutation group rotation-induced by 
  $D$}, since the concrete sets $R$, $I$, and $T$ are not
essential for our purposes. By the one-to-one relationship between rotations and
permutations, it is obvious that a polynomial diameter of the
rotation-induced group implies that the length of movement plans can
be polynomially bounded.

In the following, we will also often use arguments about possible
movements of agents in order to prove that a particular permutation
exists.

On a fully occupied strongly connected digraph, it is possible to move any agent to
any node by using the right combination of rotations, i.e., the
rotation-induced permutation group is transitive. However, this
holds only, as long as we allow for rotations on directed cycles of size
two. If we require rotations to use at least 3 nodes, then the rotation-induced
group might not be transitive any longer (see
Figure~\ref{F:non-trans}). However, then there are \emph{transitive
components} (dashed circles in
Figure~\ref{F:non-trans}), in which each agent can reach each node,
but no other nodes, i.e., these transitive components are independent
of each other.

\begin{figure}[htb]
\centering\resizebox{\ifarxiv%
0.6\columnwidth%
\else%
0.8\columnwidth\fi}{!}{
\begin{tikzpicture}[thick,main/.style = {draw, circle}]
  \node[main] (1) at (0,0) {};
  \node[main] (2) at (1,1) {};
  \node[main] (3) at (1,-1) {};
  \node[main] (4) at (2,0) {};
  \draw[dashed] (4) circle [radius=14mm,xshift=-10mm];
  \draw[->] (1) -- (2);
  \draw[->] (2) -- (3);
  \draw[->] (3) -- (4);
  \draw[->] (4) -- (1);
  \draw[->] (3) -- (1);
  \draw[->] (2) -- (4);
  \node[main] (6) at (4,0) {};
  \node[main] (7) at (5,1) {};
  \node[main] (8) at (5,-1) {};
  \draw[dashed] (6) circle [radius=14mm,xshift=10mm];
  \draw[->] (6) -- (7);
  \draw[->] (7) -- (8);
  \draw[->] (8) -- (6);
  \draw[->] (4) to [bend left=30] (6);
  \draw[->] (6) to [bend right=330] (4);
  \end{tikzpicture}
} 
\caption{Non-transitive strongly connected digraph with transitive
  components---provided rotations are disallowed for cycles of size 2}
\label{F:non-trans}
\end{figure}
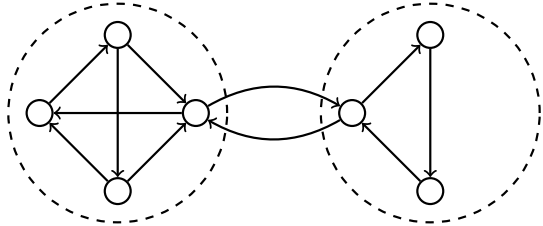

One can solve these transitive
components in isolation and then combine the respective permutations or
movement plans (as also done by Kornhauser et
al. \shortcite{kornhauser:et:al:focs-84}).
For this reason, it is enough to consider 
rotation-induced permutation groups that are transitive.

In order to show that the permutation groups rotation-induced by diMAPF
instances have polynomial diameter, we will use the following result by Driscoll and Furst
\shortcite[Theorem 3.2]{driscoll:furst:stoc-83}.

\begin{theorem}[Driscoll \& Furst]
  \label{T:diameter}
  If ${\mathbf G}$ is a primitive group containing a polynomially
  expressible 3-cycle, then the diameter of ${\mathbf G}$ is polynomially bounded.
\end{theorem}

Incidentally, if the conditions of the Theorem are satisfied,  then
$ {\mathbf G} = {\mathbf A}_n$ or ${\mathbf G} = {\mathbf S}_n$, as follows
from a Lemma that is used in Driscoll and Furst's
\shortcite[Lemma 3.4]{driscoll:furst:stoc-83} paper.

\begin{lemma}
  \label{L:jordan}
A primitive group that contains a 3-cycle is either alternating or symmetric.
\end{lemma}

It should be noted that primitiveness is implied by 2-transitivity, because
for a non-trivial block $Y$ there would exist one permutation that fixes one
element (staying in the block) and moves another element out of the
block, which contradicts that $Y$ is a block. 

\begin{proposition}
  \label{P:primitive}
  Every 2-transitive  permutation group is primitive.
\end{proposition}

In other words, it is enough to show 2-transitivity and the polynomial
expressibility of a 3-cycle in order to be able to apply
Theorem~\ref{T:diameter}, which enables us to  derive a polynomial bound for the
diameter.

Further, for demonstrating that a rotation-induced
group is the symmetric group, given 2-transitivity and a 3-cycle,
it suffices to show that the
rotation-induced group contains an odd permutation. This follows from the fact
that ${\mathbf A}_n$ contains only even permutations and
Lemma~\ref{L:jordan}. 
Note that in case the permutation group is the
symmetric group, it means that the diMAPF instance is
always solvable, a fact we will later use in the proof of Lemma~\ref{L:hybrid:movements}.

\subsection{2-Transitivity}

Almost all transitive permutation groups
rotation-induced by diMAPF instances on strongly connected digraphs are
2-transitive, as shown next. Intuitively, it means that we can move
any two agents to any two places in the digraph---moving perhaps
other agents around as well.

\begin{lemma}
  \label{L:2trans}
  Transitive permutation groups rotation-induced by strongly connected
  non-trivial digraphs
  that are not partially bidirectional cycle graphs are
  2-transitive.
\end{lemma}
\begin{proof}
  In order to prove this lemma, we will show that for two fixed nodes
  $x$ and $y$, it is possible to move any pair of agents $a_u$ and
  $a_v$ from the node pair $(u,v)$ to the node pair $(x,y)$. This
  implies that we also could move any pair $(a_{u'}, a_{v'})$ from
  $(u',v')$ to $(x,y)$. Composing the first plan with the inverse of
  the second plan means that we can move agents from any pair of nodes
  $(u,v)$ to any other pair of nodes $(u',v')$, which means the group
  is 2-transitive.
  
  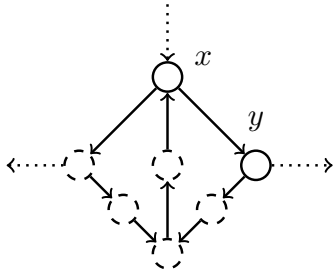
\begin{figure}[htb]
    \centering
    \resizebox{\ifarxiv%
 0.4\columnwidth%
\else%
0.5\columnwidth\fi}{!}{
  \begin{tikzpicture}[thick, main/.style = {draw, circle}] 
\node[main,dashed] (2) at (1,0) {};
\node[style={draw,dashed,circle}] (x0) at (1.5,-0.5) {};
\node[] (3) at (2,2) {}; 
\node[main] (4) at (2,1) {};
\node (u) [above right of=4,yshift=-5mm,xshift=-3mm] {$x$};
\node[style={draw,dashed,circle}] (5) at (2,-1) {};
\node[style={draw,dashed,circle}] (z0)  at (2,0) {};
\node[main] (7) at (3,0){};
\node (y) [above of=7,yshift=-5mm] {$y$};
\node[style={draw,dashed,circle}] (y0) at (2.5,-0.5) {};
\node (9) at (4,0){};
\draw[<-,dotted] (1) -- (2); 
\draw[<-] (2) -- (4); 
\draw[->] (2) -- (x0); 
\draw[->] (x0) -- (5); 
\draw[<-,dotted] (4) -- (3); 
\draw[->] (5) -- (z0); 
\draw[->] (z0) -- (4); 
\draw[->] (4) -- (7);
\draw[->] (7) -- (y0);
\draw[->] (y0) -- (5);
\draw[->,dotted] (7) -- (9);
\end{tikzpicture}
}
\caption{Strongly connected component containing at
  least two cycles: Demonstrating 2-transitivity}
\label{F:SBC:2trans}
\end{figure}

  By Proposition~\ref{P:normalize:SBC}, the digraph must contain at least two
  cycles, both containing at least 2 nodes, where the left cycle may
  consist of only shared nodes, 
  as depicted in
  Figure~\ref{F:SBC:2trans}.
  The dashed nodes signify possible
  additional nodes.
  The dotted arcs exemplify potential connections to
  other nodes in the digraph. 

  By transitivity, we can
  move any agent $a_u$ from node $u$ to node $x$.
  After that, we can move any agent
  $a_v$ from node $v$ to node $y$. This may lead to rotating the
  agent $a_u$ out of the left cycle. In order to prohibit that, we modify the movement plan as
  follows. As long as $a_v$ has
  not entered one of the two cycles yet, every time $a_u$ is
  threatened to be rotated out of the left cycle in the next move, we
  rotate the entire left cycle so that $a_u$ is moved to a node that
  will not lead to rotating $a_u$ out of the left cycle.
  
  If $a_v$ arrives in the right cycle (including the shared nodes), we rotate the right
  cycle iteratively. Whenever $a_u$ is placed on $x$ and $a_v$ has not
  yet arrived at $y$, we rotate on the left cycle. Otherwise, we stop
  and are done. When $a_v$ is
  placed on $x$, then in the next move, we rotate the right cycle and
  $a_v$ is placed on $y$. After that, we can rotate the left cycle
  until $a_u$ arrives at $x$. This is possible, because $a_u$ never
  leaves the left cycle.

  If $a_v$  arrives on nodes not belonging to the right cycle, we
  rotate the left cycle. When $a_v$ arrives at $x$, we rotate right and
  then we rotate left until $a_u$ arrives again at $x$.
\end{proof}

\subsection{3-Cycles}

The construction of
3-cycles will be shown by a case analysis over
the possible forms of two connected cycles. By
Proposition~\ref{P:normalize:SBC}, we know that every strongly
connected non-trivial digraph that is not a partially bidirectional cycle contains a subgraph as shown in
Figure~\ref{F:SBC:normal}.
  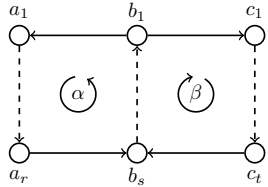
\begin{figure}[htb]
    \centering
\resizebox{\ifarxiv%
0.3\columnwidth%
\else%
0.5\columnwidth\fi}{!}{
\begin{tikzpicture}[thick, main/.style = {draw, circle}] 
  \node[main] (a1) at (0,2) {};
  \node (la1) [above of=a1,yshift=-6mm] {$a_1$};
  \node[main] (an1) at (0,0) {};
  \node (lan1) [below of=an1,yshift=6mm] {$a_{r}$};
  \node[main] (b1) at (2,2) {};
  \node (lb1) [above of=b1,yshift=-6mm] {$b_1$};
  \node[main] (bn2) at (2,0) {};
  \node (lbn2) [below of=bn2,yshift=6mm] {$b_{s}$};  
  \node (alpha) at (1,1) {$\alpha$};
  \circledarrow{black,thick}{alpha}{0.3cm};
  \node[main] (c1) at (4,2) {};
  \node (lc1) [above of=c1,yshift=-6mm] {$c_1$};
  \node[main] (cn3) at (4,0) {};
  \node (lcn3) [below of=cn3,yshift=6mm] {$c_{t}$};
  \node (beta) at (3,1) {$\beta$};
  \backcircledarrow{black,thick}{beta}{0.3cm};
   \draw[<-] (a1) -- (b1); 
  \draw[->,dashed] (a1) -- (an1);
  \draw[->] (an1) -- (bn2); 
  \draw[->,dashed] (bn2) -- (b1);
  \draw[->] (b1) -- (c1);
  \draw[->,dashed] (c1) -- (cn3);
  \draw[->] (cn3) -- (bn2);
\end{tikzpicture}}
\caption{Strongly connected component consisting of at
  least two cycles inducing two permutations}
\label{F:SBC:normal}
\end{figure}

We characterize such connected cycles by the three parameters $(r, s,
t)$ and will talk about \emph{cycle pairs of type $(r, s,
  t)$}, assuming wlg.{} $r \leq t$.
Below, we will show that for almost all cycle pairs
one can construct a
3-cycle, save for cycle pairs of type $(2,2,2)$ and $(1,3,2)$. These are the directed counterparts of
Kornhauser et al.'s \shortcite{kornhauser:et:al:focs-84} $T_0$-
and Wilson's \shortcite{wilson:jct-74} $\theta_0$-graph. We will
therefore call such cycle pairs \emph{$T_0$-pairs}.

\begin{lemma}
  \label{L:3cycle}
  Each transitive permutation group rotation-induced by a cycle pair that is not a $T_0$-pair
  contains a polynomially expressible 3-cycle.
\end{lemma}

\begin{proof}
  The 3-cycles will be constructed
  from 
  $\alpha = (a_1 \ldots a_{r} b_{s} \ldots b_1)$ and
  $\beta = (c_1 \ldots c_{t} b_{s} \ldots b_1)$.  We prove the
  claim by case analysis over the parameters $(r, s, t)$ (see Fig.~\ref{F:SBC:normal}).
  \begin{description}
  \item[$(0,\_,\_)$:] 
    This implies   by
  Proposition~\ref{P:normalize:SBC} that $s \geq 3$ and
  $t \geq 1$ and we have
    $\beta\ \alpha^{-1}\ \beta^{-1}\ \alpha = (b_{s}\ c_{t}\
    b_1)$ as a 3-cycle.
  \item[$(\geq 1,1,\_)$:]  
    In this case, the same expression 
    delivers a slightly different 3-cycle:  $\beta\ \alpha^{-1}\ \beta^{-1}\ \alpha = (b_{1}\ a_{1}\
    c_{t})$.
    
  \item[$(1,2,\_)$:] 
    $\alpha$ is a 3-cycle.
  \item[$(1,\geq 3,1)$:]  
    $\alpha^{-1}\beta = (b_{s}\ a_1\ c_1)$ is the desired 3-cycle.
  \item[$(1,3,2)$:]  This is a $T_0$-pair, so there is nothing to
    prove here. It can be shown by exhaustive enumeration that
    the rotation-induced permutation group does not contain 3-cycles. 
  \item[$(1,3,\geq 3)$:]  
    $\beta\ \alpha^{-1}\ \beta^{-1}\
    \alpha = (b_{s}\ c_{t}) (b_1\
    a_{1})$. Consider now $\chi = \beta^2\ (\alpha^{-1}\
    \beta)^{2}\ \beta^{-2}$. This permutation fixes $b_{s}$ and
    $c_{t}$ and moves $c_{t-2}$ to $a_1$ and $a_1$ to $b_1$
    while also moving other things around. This means:
    \begin{align*}
      (\beta\ \alpha^{-1}\ \beta^{-1}\ \alpha)^{\chi^{-1}} & =
                                                             \chi\
                                                             \beta\
                                                             \alpha^{-1}\
                                                             \beta^{-1}\
                                                             \alpha\
                                                             \chi^{-1}
      \\
                                                           & = \chi (b_{s}\ c_{t}) (b_1\ a_{1}) \chi^{-1} \\
                                                           &=
                                                             (b_{s}\
                                                             c_{t})
                                                             (a_1\
                                                             c_{t-2}).
    \end{align*}
    Composing the result with the original permutation is now what
    results in a 3-cycle.\footnote{This construction is similar to one
      used by Kornhauser \shortcite{kornhauser:tr-84} in the
      proof of Theorem~1 for $T_2$-graphs.} That is, $\lambda = (\beta\ \alpha^{-1}\
    \beta^{-1}\ \alpha)^{\epsilon+\chi^{-1}}$ is the permutation, we looked for:
    \begin{align*}
      \lambda &= (\beta\ \alpha^{-1}\ \beta^{-1}\
                \alpha)^{\epsilon+\chi^{-1}} \\
      & = ((b_{s}\ c_{t}) (b_1\    a_{1})) \circ ((b_{s}\
        c_{t}) (b_1\    a_{1}))^{\chi^{-1}} \\
      & = ((b_{s}\ c_{t}) (b_1\    a_{1})) \circ ((b_{s}\
        c_{t}) (a_1\ c_{t-2})) \\
      & = (b_1\   c_{t-2}\ a_1).
    \end{align*}
    \item[$(1,\geq 4,\geq 2)$:]  
      In this case, $\xi = (\alpha\ \beta^{-1}\ \alpha^{-1}\
      \beta)^{\beta(\epsilon+\alpha^{-2})}$ is the claimed
      3-cycle:\footnote{This construction of a permutation as well as
        the ones for the cases further down are
        similar to one that has been used in a similar context by
        Bachor et al. \shortcite{bachor:et:al:aaai-23}.}
      \begin{align*}
        \xi & = (\alpha\ \beta^{-1}\ \alpha^{-1}\
              \beta)^{\beta(\epsilon+\alpha^{-2})} \\
        & = \left((b_{s}\ a_{r})(b_1\
          c_1)\right)^{\beta(\epsilon+\alpha^{-2})} \\
        & = \left((b_{s-1}\ a_{r})(c_1 \
          c_2)\right)^{\epsilon+\alpha^{-2}} \\
        & = ((b_{s-1}\ a_{r})(c_1 \
          c_2)) \circ ((b_{s-1}\ a_{r})(c_1 \
          c_2))^{\alpha^{-2}} \\
            & = ((b_{s-1}\ a_{r})(c_1 \
          c_2)) \circ ((b_2\ a_{r})(c_1 \
          c_2)) \\
            & = (b_{s-1}\ b_2\ a_{r}).
      \end{align*}
  \item[$(2,2,2)$:] This is the other case that is excluded in the
    claim and the same comment as in case $(1,3,2)$ applies.
  \item[$(2,\geq 3,2)$:]  
    For this case, the sought 3-cycle is $\zeta = (\beta\
    \alpha^{-1}\ \beta^{-1}\ \alpha)^{\alpha(\epsilon+\beta^{-2})}$:
      \begin{align*}
        \zeta &= (\beta\
                \alpha^{-1}\ \beta^{-1}\ \alpha)^{\alpha(\epsilon+\beta^{-2})} \\
              & = ((b_{s}\ c_{t}) (a_1\ b_1))^{\alpha(\epsilon+\beta^{-2})} \\
              &= ((b_{s-1}\ c_{t}) (a_1\ a_2))^{\epsilon+\beta^{-2}} \\
              & = ((b_{s-1}\ c_{t}) (a_1\ a_2)) \circ ((b_{s-1}\
                c_{t}) (a_1\ a_2))^{\beta^{-2}} \\
                & = ((b_{s-1}\ c_{t}) (a_1\ a_2)) \circ ((b_{1}\
                  c_{t}) (a_1\ a_2)) \\
                  & = (b_{s-1}\ b_1\   c_{t}).
        \end{align*}
      
      \item[$(\geq 2,\geq 2,\geq 3)$:]  
        Interestingly, the above product of basic permutations works for the general
        case, when a cycle pair is ``large enough,'' as well. Because of the
        different structure of the cycle pairs, the result is slightly
        different, though (differences are underlined):
      \begin{align*}
        \zeta &= (\beta\
                \alpha^{-1}\ \beta^{-1}\ \alpha)^{\alpha(\epsilon+\beta^{-2})} \\
              & = ((b_{s}\ c_{t}) (a_1\ b_1))^{\alpha(\epsilon+\beta^{-2})} \\
              &= ((b_{s-1}\ c_{t}) (a_1\ a_2))^{\epsilon+\beta^{-2}} \\
              & = ((b_{s-1}\ c_{t}) (a_1\ a_2)) \circ ((b_{s-1}\
                c_{t}) (a_1\ a_2))^{\beta^{-2}} \\
                & = ((b_{s-1}\ c_{t}) (a_1\ a_2)) \circ ((\underline{c_{t-2}}\
                  c_{t}) (a_1\ a_2)) \\
                  & = (b_{s-1}\  \underline{c_{t-2}}\   c_{t}).
        \end{align*}

      \end{description}
      This covers all possible cases. Note that when taking $\alpha$
      and $\beta$ as generators, then the inverses $\alpha^{-1}$ and
      $\beta^{-1}$ can be expressed by linearly many products. Since
      the expressions for all cases have constant length, in all cases
      the 3-cycles are linearly
      expressible. 
      So, the claim holds.
\end{proof}

In order to be able to  do away with $T_0$-pairs, we will assume that our digraphs
contain at least seven nodes.  For all smaller digraphs,  the
diameter of the rotation-induced permutation group is constant. One only has then
to show that strongly connected digraphs with seven or more nodes
admit for the generation of a 3-cycle. 

\begin{lemma}
  \label{L:not:t0}
  Each transitive permutation group rotation-induced by a strongly connected
  digraph with at least 7 nodes that is not a partially bidirectional cycle 
  contains a polynomially
  expressible 3-cycle.
\end{lemma}

\begin{proof}
  By Lemma~\ref{L:3cycle}, it is enough to prove the claim for
  digraphs that contain a $T_0$-pair. In order to do so, all possible
  extensions of $T_0$-pairs have to be analyzed. It turns out that it
  is sufficient to consider all extensions with one additional ear,
  since for all cases with one additional ear, we are able to identify
  a 3-cycle. In order to show that, we first check whether a new cycle
  pair is
  created that is not a $T_0$-pair, in which case Lemma~\ref{L:3cycle}
  is applicable. If the newly created cycle pairs are all $T_0$-pairs,
  one has to demonstrate that by the addition of the ear a new
  permutation is added that can be used to create a 3-cycle.

  Because the longest ear in $T_0$-pairs has a length of 2, we
  consider ears up to length 2. Adding a longer ear would result in a
  pair of type  $(\_,\_,\geq 3)$, which admits a
  3-cycle according to Lemma~\ref{L:3cycle}.  Since the $T_0$-pairs
  contain six nodes each, there are two different $T_0$ pairs, and we
  consider ears of length one and two,
  we need to analyze $6^2 \times 2 \times 2=144$ cases. This has
  been done using a \emph{SageMath} \cite{sagemath} script,
  which is listed in the appendix. This script identified three
  non-isomorphic extensions of the $(1,3,2)$- and $(2,2,2)$-type cycle
  pairs that contain only $T_0$-pairs as cycle pairs. These are shown
  in Figure~\ref{F:t0ext}.
 \begin{figure}[htb]
 \centering
\resizebox{\ifarxiv%
0.8\columnwidth%
\else%
\columnwidth\fi}{!}{
\begin{tabular}{cc}
  \multicolumn{2}{c}{
    \begin{tikzpicture}[thick, main/.style = {draw, circle}] 
  \node[main] (a1) at (0,0.5) {};
  \node (la1) [above of=a1,yshift=-6mm] {$a_1$};
  \node[main] (b1) at (2,2) {};
  \node (lb1) [above of=b1,yshift=-6mm] {$b_1$};
  \node[main] (b2) at (2,0.5) {};
  \node (lb2) [above left of=b2,xshift=4mm,yshift=-4mm] {$b_2$};
  \node[main] (b3) at (2,-1) {};
  \node (lb3) [below of=b3,yshift=6mm] {$b_3$};
  \node (alpha) at (1,0.5) {$\alpha$};
  \circledarrow{black,thick}{alpha}{0.3cm};
  \node[main] (c1) at (4,2) {};
  \node (lc1) [above of=c1,yshift=-6mm] {$c_1$};
  \node[main] (c2) at (4,-1) {};
  \node (lc2) [right of=c2,xshift=-6mm] {$c_2$};
  \node (beta) at (3,1) {$\beta$};
  \backcircledarrow{black,thick}{beta}{0.3cm};
   \draw[<-] (a1) -- (b1); 
  \draw[->] (a1) -- (b3); 
  \draw[->] (b3) -- (b2);
  \draw[->] (b2) -- (b1);
  \draw[->] (b1) -- (c1);
  \draw[->] (c1) -- (c2);
  \draw[->] (c2) -- (b3);
  \node[main] (d1) at (3,-0.25) {};
  \node (ld1) [above of=d1,yshift=-6mm] {$e_1$};
  \node (gamma) at (2.5,-0.6) {$\gamma$};
  \backcircledarrow{black,thick}{gamma}{0.3cm};
  \draw[->,dashed] (b2) -- (d1);
  \draw[->,dashed] (d1) -- (c2);
  \node (a) [below of=b3] {(a)};
\end{tikzpicture}} \\
\begin{tikzpicture}[thick, main/.style = {draw, circle}] 
  \node[main] (a1) at (0,2) {};
  \node (la1) [above of=a1,yshift=-6mm] {$a_1$};
  \node[main] (an1) at (0,0) {};
  \node (lan1) [below of=an1,yshift=6mm] {$a_2$};
  \node[main] (b1) at (2,2) {};
  \node (lb1) [above of=b1,yshift=-6mm] {$b_1$};
  \node[main] (bn2) at (2,0) {};
  \node (lbn2) [below of=bn2,yshift=6mm] {$b_2$};  
  \node (alpha) at (1,1) {$\alpha$};
  \circledarrow{black,thick}{alpha}{0.3cm};
  \node[main] (c1) at (4,2) {};
  \node (lc1) [above of=c1,yshift=-6mm] {$c_1$};
  \node[main] (cn3) at (4,0) {};
  \node (lcn3) [below of=cn3,yshift=6mm] {$c_2$};
  \node (beta) at (3,1) {$\beta$};
  \backcircledarrow{black,thick}{beta}{0.3cm};
   \draw[<-] (a1) -- (b1); 
  \draw[->] (a1) -- (an1);
  \draw[->] (an1) -- (bn2); 
  \draw[->] (bn2) -- (b1);
  \draw[->] (b1) -- (c1);
  \draw[->] (c1) -- (cn3);
  \draw[->] (cn3) -- (bn2);
  \node[main] (d2) at (-1,0) {};
  \node (ld2) [below of=d2,yshift=6mm] {$e_2$};
  \node[main] (d1) at (-1,2) {};
  \node (ld1) [above of=d1,yshift=-6mm] {$e_1$};
  \node (delta) at (-0.5,1) {$\delta$};
  \backcircledarrow{black,thick}{delta}{0.3cm};
  \draw[->,dashed] (an1) -- (d2);
  \draw[->,dashed] (d2) -- (d1);
  \draw[->,dashed] (d1) -- (a1);
  \node (b) [below of=bn2] {(b)};
\end{tikzpicture}
&
\begin{tikzpicture}[thick, main/.style = {draw, circle}] 
  \node[main] (a1) at (0,2) {};
  \node (la1) [above of=a1,yshift=-6mm] {$a_1$};
  \node[main] (an1) at (0,0) {};
  \node (lan1) [below of=an1,yshift=6mm] {$a_2$};
  \node[main] (b1) at (2,2) {};
  \node (lb1) [above of=b1,yshift=-6mm] {$b_1$};
  \node[main] (bn2) at (2,0) {};
  \node (lbn2) [below of=bn2,yshift=6mm] {$b_2$};  
  \node (alpha) at (1,1) {$\alpha$};
  \circledarrow{black,thick}{alpha}{0.3cm};
  \node[main] (c1) at (4,2) {};
  \node (lc1) [right of=c1,xshift=-6mm] {$c_1$};
  \node[main] (cn3) at (4,0) {};
  \node (lcn3) [right of=cn3,xshift=-6mm] {$c_2$};
  \node (beta) at (3,1) {$\beta$};
  \backcircledarrow{black,thick}{beta}{0.3cm};
   \draw[<-] (a1) -- (b1); 
  \draw[->] (a1) -- (an1);
  \draw[->] (an1) -- (bn2); 
  \draw[->] (bn2) -- (b1);
  \draw[->] (b1) -- (c1);
  \draw[->] (c1) -- (cn3);
  \draw[->] (cn3) -- (bn2);
  \node[main] (d2) at (5,-0.5) {};
  \node (ld2) [below of=d2,yshift=6mm] {$e_2$};
  \node[main] (d1) at (5,2.5) {};
  \node (ld1) [above of=d1,yshift=-6mm] {$e_1$};
  \node (eta) at (4.5,1) {$\eta$};
  \backcircledarrow{black,thick}{eta}{0.3cm};
  \draw[<-,dashed] (d2) -- (d1);
  \draw[<-,dashed] (d1) -- (b1);
    \draw[<-,dashed] (bn2) -- (d2);
  \node (c) [below of=bn2] {(c)};
\end{tikzpicture}
\end{tabular}
}
\caption{Extensions of $T_0$-pairs containing only $T_0$-pairs}
\label{F:t0ext}
\end{figure}
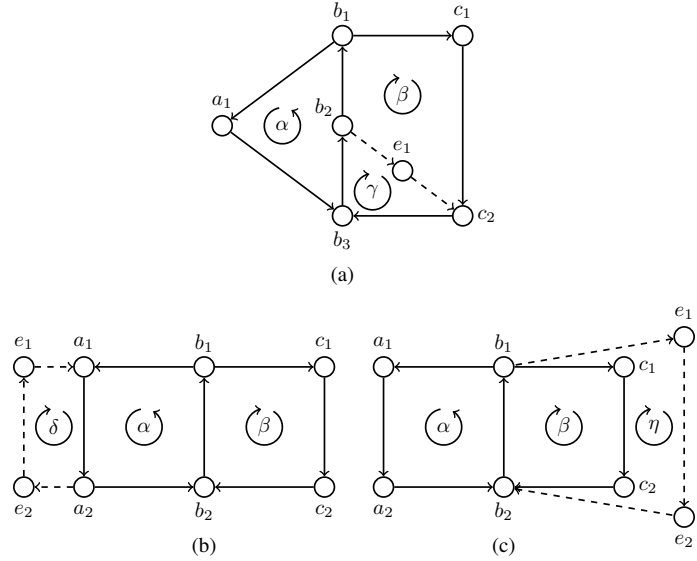

It is now an easy exercise to identify 3-cycles for these cases:
\begin{description}
\item[(a):]  $\beta\ \alpha\ \beta^{-1}\ \gamma^{-1}$,
\item[(b):]  $\beta^{-1}\ \delta^{-1}\ \alpha^{-1}\ \beta^2\ \delta^{-1}\ \alpha^{-1}\ \delta^{-1}$,
\item[(c):] $\alpha\ \beta\ \eta\ \alpha^{-1}\ \beta^{-1}\ \eta^{-1}$.
\end{description}
So, the claim holds for all cases.
\end{proof}

The above results enable us now to prove the claim that the rotation-induced
permutation groups have a polynomial diameter.

\begin{lemma}
  \label{L:poly:diameter}
  Each transitive permutation group rotation-induced by a strongly connected
  digraph  has a polynomial diameter.
\end{lemma}

\begin{proof}
  We prove the claim by case analysis. Let $n$ be the number of nodes.
  \begin{enumerate}
  \item \emph{$n < 7$:} There exist only
    finitely many permutation groups rotation-induced by such graphs. The diameter is
    therefore $O(1)$ in this case.

  \item \emph{$n \geq 7$:}
    \begin{enumerate}
          \item \emph{The digraph is a partially bidirectional
              cycle of $n$ nodes:} If rotations of size 2 are
            disallowed, then each possible permutation can be expressed by
            at most $O(n)$ compositions of the permutation
            corresponding to the rotation of all
            agents by one place. If rotations of size 2 are permitted
            and the cycle has at least one backward-pointing arc, 
            the rotation-induced group is ${\mathbf S}_n$ and any
            cyclic order is achievable using $O(n^2)$ swaps and
            rotations on the cycle, similar
            to the way bubble sort works.
\item
  \emph{The digraph is a strongly connected digraph that is not
    a partially bidirectional cycle}: For this case, we use
  Theorem~\ref{T:diameter}, i.e., 
  it is enough to show that a permutation group is 2-transitive and contains a
  polynomially expressible
  3-cycle. By Proposition~\ref{P:normalize:SBC}, we know that the
  digraph contains a cycle pair. Now, 2-transitivity follows from
  Lemma~\ref{L:2trans}. The existence of a polynomially expressible 3-cycle follows from
  Lemma~\ref{L:not:t0}. So, in this case, the
  claim holds as well. 
\end{enumerate}
\end{enumerate}
This covers all possible cases, so the claim holds.
\end{proof}

With that, another partial
result on the small solution hypothesis follows.

\begin{theorem}
  \label{T:ssh:rotation}
The {\em Small Solution Hypothesis for diMAPF on strongly connected
  digraphs} is true, provided only rotations are allowed.
\end{theorem}

\begin{proof}
Decompose the graph into its transitive components, i.e., the components
on which the rotation-induced permutation group is transitive. The only way to
solve such an instance is to solve the instance for each
transitive component in isolation and afterward combine the
result---if possible. Each transitive component has a polynomial
solution because by Lemma~\ref{L:poly:diameter} each rotation-induced
group has polynomial diameter, so the combined solution will be
polynomial.  
\end{proof}

\section{DiMAPF: The General Case}

Finally, we will consider the case that simple moves as well as
rotations are permitted. In order to be able to apply permutation
group theory, we will initially restrict ourselves to diMAPF instances
on strongly connected digraphs $\langle D,R,I,T\rangle$ such that the
set of occupied nodes is identical in the initial and the goal state,
i.e., $I(R) = T(R)$, which can be viewed as permutations on the set of
nodes $V$.  This restriction is non-essential since one can
polynomially transform a general diMAPF instance to such a restricted
instance, as shown in Corollary~\ref{C:hybrid:reduction} below.

\begin{lemma}
  \label{L:hybrid:reduction}
  Given a diMAPF instance $\langle D,R,I,T\rangle$, with $D$ a
  strongly connected digraph, an instance  $\langle D,R,I,T'\rangle$
  can be computed in polynomial time such that $I(R) = T'(R)$, and
  $\langle D,R,T,T'\rangle$ and $\langle D,R,T',T\rangle$ are both
  solvable using plans of polynomial length.
\end{lemma}

\begin{proof}
  In order to construct $\langle D,R,I,T'\rangle$, generate an
  arbitrary mapping from blanks in $T$, the
  \emph{source nodes}, to blanks in $I$, the {\em
    target nodes}. Then ``move'' 
  the blanks from the source nodes to the target nodes (against the
  direction of the arcs) by moving the appropriate agents. This is always
  possible because $D$ is strongly connected.

  The new configuration is $T'$ and clearly $T'(R)=I(R)$. Further,
  $T'$ is obviously reachable from $T$ in at most $O(n^2)$ simple moves, $n$
  being the number of nodes.

  Reaching $T$ from $T'$ is possible by undoing each movement of a
  blank in the opposite order. Undoing such a movement can be done
  by applying Proposition~\ref{P:inverse:move} iteratively resulting
  in $O(n^4)$ simple moves.
\end{proof}

Since $T$ is reachable from $T'$ and vice versa using only polynomial
many steps, the next corollary follows immediately.

\begin{corollary}
  \label{C:hybrid:reduction}
  Let $\langle D,R,I,T\rangle$ and $\langle D,R,I,T'\rangle$ be as in
  Lemma~\ref{L:hybrid:reduction}. Then   $\langle D,R,I,T\rangle$ is solvable
  with a polynomial plan if and only if $\langle D,R,I,T'\rangle$ is
  solvable with a polynomial plan.
\end{corollary}

As in the previous section, when only rotations were permitted,
we will again talk about \emph{induced permutation groups}, however, now they
are induced by simple moves as well as rotations. Again, we will only
consider transitive components of strongly connected digraphs,
i.e., components that induce a transitive group. 

\begin{lemma}
  \label{L:hybrid:movements}
  Each transitive permutation group induced by a strongly connected
  digraph $D$ (using rotations and simple moves) has a polynomial diameter. 
\end{lemma}

\begin{proof}
  We make the assumptions that we have at least one blank---because
  otherwise Lemma~\ref{L:poly:diameter} applies---and that the digraph
  has at least 7 nodes---since smaller groups have obviously a diameter
  of $O(1)$.
  
  Assuming that rotations on cycles with only 2 nodes are permitted or
  that the graph does not contain such cycles,
  we reduce this problem to the case where only rotations are
  allowed.

  If all cycles have odd length, then using Lemma~\ref{L:hybrid:reduction},
  move one blank to a node that is a non-articulation node. Such a
  node must exist. Make sure that this blank will always be blank
  after emulating rotations on not fully occupied cycles. This will
  not destroy transitivity, and it effectively introduces at least one
  rotation of even length (corresponding to an odd permutation).
  Consider now all remaining blanks as ``virtual agents.''  Now each
  possible synchronous rotation on a cycle containing such virtual
  agents or the fixed blank can be emulated by a sequence of simple moves on this cycle.
  Applying Lemmas~\ref{L:not:t0} and~\ref{L:2trans} gives us together
  with Theorem~\ref{T:diameter} a polynomial diameter. With the
  presence of an odd permutation and Lemma~\ref{L:jordan}, the induced
  group must be symmetric, i.e., the instance is solvable for all
  $\langle I,T\rangle$ pairs.

  Let us now assume that rotations on cycles with 2 nodes are not
  permitted, and that the graph contains such cycles.  If the
  underlying graph ${\cal G}(D)$ is a tree, then rotations are impossible, and
  simple moves alone are enough, i.e., we can rely on Theorem 2 by
  Kornhauser et al. \shortcite{kornhauser:et:al:focs-84}.  Otherwise,
  the underlying graph is separable and contains strongly connected
  components inducing a tree-like structure. One can then show
  $k$-transitivity ($k$ being the number of agents) as Kornhauser et
  al. \shortcite[Section 2.3.1.1]{kornhauser:et:al:focs-84} did, which
  implies that the induced group is symmetric and has a polynomial
  diameter.

  This means the claim follows in all possible cases.
\end{proof}

This implies that the \emph{Small Solution Hypothesis} is true for all
possible combinations of movements.

\begin{theorem}
  \label{T:ssh:hybrid}
The {\em Small Solution Hypothesis for diMAPF on strongly connected
  digraphs} is true regardless of whether simple moves or rotations
are allowed.
\end{theorem}

\begin{proof}
  If only simple moves are
  possible, Theorem~\ref{T:ssh:simple} applies. If only
  synchronous rotations are possible, the claim follows from
  Theorem~\ref{T:ssh:rotation}. In case, simple moves and rotations
  are possible, the claim follows from
  Lemma~\ref{L:hybrid:movements} and the same arguments about combining
  the solutions of transitive components as used in the
  proof of Theorem~\ref{T:ssh:rotation}.
\end{proof}

As mentioned above, this result enables us to finally settle the
question of the computational complexity of diMAPF.
Using Theorem 4
from the paper establishing NP-hardness of diMAPF 
\cite{nebel:icaps-20}, the next Theorem is immediate.

\begin{theorem}
  DiMAPF is NP-complete, even when synchronous rotations are possible.
\end{theorem}

\section{Conclusion and Outlook}

This paper provides an answer to an open question about the
computational complexity of the multi-agent pathfinding problem on
directed graphs. Together with the results from an earlier paper
\cite{nebel:icaps-20}, we can conclude that diMAPF is NP-complete,
even when synchronous rotations are permitted.  

While the result might have only a limited impact on practical
applications, it nevertheless provides some surprising insights.
First, it shows that if only simple moves are permitted, then the
inverse of such a move can be polynomially synthesized, meaning that
agents can move against the direction of an arc with only polynomial overhead. This is something
that should have been obvious to everybody, but apparently was missed.
Second, it shows that permutation group theory is applicable to the
analysis of diMAPF, something that was not recognized previously
\cite{botea:et:al:jair-18}. As it turned out, there are quite a number
of differences to the undirected case, though, e.g., Kornhauser et
al.'s \shortcite{kornhauser:et:al:focs-84} Lemma 1 is not valid in the
directed case and there are two counterparts to the
$T_0$-graph.  Third, although only implicitly, the results show
that the special case of diMAPF on strongly connected digraphs is a
polynomial-time problem. One needs to identify the transitive
components, though, which although not necessarily straight-forward
\cite[Section 3.1]{dewilde:et:al:jair-14}, is
always a polynomial problem. Alternatively, one might be able to  adapt the
feasibility check by Auletta et al. \shortcite{auletta:et:al:algorithmica-99}.
Fourth, this result provides an answer to a question about the generalization of the \emph{robot movement
problem} \cite{papadimitriou:et:al:focs-94} to directed graphs with a
variable number of robots \cite{wu:grumbach:dam-10}.
For a variable number of
mobile obstacles and mobile robots, the problem is NP-complete in the 
general case. For strongly connected digraphs, the above
mentioned polynomiality of diMAPF on strongly connected graphs carries
over to the robot  movement problem.
\ifarxiv
\newpage
\fi
\appendix
\section{SageMath Script for the Proof of
  Lemma~\ref{L:not:t0}\footnotemark}
\footnotetext{You find this SageMath script and others
  related to this paper at \url{https://github.com/BernhardNebel/small-solution-hypothesis}.}
\begin{footnotesize}
\begin{verbatim}
def shared(c1, c2):
  return len(set(c1) & set(c2))
def ptype(c1, c2):
  if len(c1) > len(c2): c1,c2 = c2,c1
  if c1 != c2 and shared(c1,c2) > 0: 
    return (len(c1)-shared(c1,c2)-1, 
              shared(c1,c2),
              len(c2)-shared(c1,c2)-1)
def t0pairs(dig):
  for c1 in dig.all_simple_cycles():
    for c2 in dig.all_simple_cycles():
      if ptype(c1,c2) not in \
        [(2,2,2),(1,3,2),None]: return
  return True
t0a={"a1":["b3"],"b3":["b2"],"b2":["b1"], 
   "b1":["a1","c1"],"c1":["c2"],"c2":["b3"]}
t0b={"a1":["a2"],"a2":["b2" ],"b2":[ "b1"], 
   "b1":["a1","c1"],"c1":["c2"],"c2":["b2"]}
ears = [["e1"],["e1","e2"]]
tested = []
for g, p in ((t0a, (1,3,2)), (t0b,(2,2,2))):
  for h in g.keys():
    for t in g.keys():
      for e in ears:
        t0 = DiGraph(g)
        t0.add_path([h]+e+[t])
        if all([not t0.is_isomorphic(d) \
                 for d in tested]):
          tested += [t0]
          if t0pairs(t0): print(p,[h]+e+[t])
\end{verbatim}
\end{footnotesize}
\ifarxiv
\newpage
\fi
\bibliography{ssh}
\ifarxiv
\bibliographystyle{abbrv}
\else
\fi

\end{document}